\documentclass[sn-standardnature]{sn-jnl}

\usepackage{lineno,hyperref}
\usepackage{color}
\usepackage{cmll}
\usepackage{graphicx}
\usepackage{mathtext}
\usepackage{amssymb,amsmath,latexsym,amsfonts,array,epsfig}
\usepackage{mathrsfs}

\jyear{2021}%

\theoremstyle{thmstyleone}%
\newtheorem{theorem}{Theorem}
\newtheorem{proposition}[theorem]{Proposition}%

\theoremstyle{thmstyletwo}%

\theoremstyle{thmstylethree}%
\newtheorem{definition}{Definition}%

\begin{document}

\title[Goal Recognition by a Minimally Pre-Trained System etc.]{Object Recognition by a Minimally Pre-Trained System in the Process of Environment Exploration.}

\author*[1]{\fnm{First} \sur{Dmitry Maximov}}\email{dmmax@inbox.ru}

\author[1]{\fnm{Second} \sur{Sekou A. K. Diane}}\email{diane1990@yandex.ru}

\affil*[1]{\orgdiv{Lab 29}, \orgname{Trapeznikov Institute of Control Science Russian Academy of Sciences}, \orgaddress{\street{65, Profsoyuznaya st.}, \city{Moscow}, \postcode{117997}, \country{Russia}}}

\abstract{
We update the method of describing and assessing the process of the study of an abstract environment by a system, proposed earlier. We do not model any biological cognition mechanisms and consider the system as an agent equipped with an information processor (or a group of such agents),  which makes a move in the environment, consumes information supplied by the environment, and gives out the next move (hence, the process is considered as a game). The system moves in an unknown environment and should recognize new objects located in it. In this case, the system should build comprehensive images of visible things and memorize them if necessary (and it should also choose the current goal set). The main problems here are object recognition, and the informational reward rating in the game. Thus, the main novelty of the paper is a new method of evaluating the amount of visual information about the object as the reward. In such a system, we suggest using a minimally pre-trained neural network to be responsible for the recognition: at first, we train the network only for Biederman geons (geometrical primitives). The geons are generated programmatically and we demonstrate that such a trained network recognizes geons in real objects quite well. We also offer to generate, procedurally, new objects from geon schemes (geon combinations in images) obtained from the environment and to store them in a database. In this case, we do not obtain new information about an object (i.e., our reward is maximal, thus the game and the object cognition process stop) when we stop getting new schemes of this kind. These schemes are generated from geons connected with the object. In the case of a possibly known item, the informational reward is maximal when we have no more detection uncertainty for any of the objects.}


\maketitle

\section{Introduction}
Universal artificial intelligence (UAI) is a unifying framework and a general formal foundational
theory for artificial intelligence investigations \cite{Hutt18}. Its primary goal is to give a mathematical answer to the question: what is the right thing to do in an unknown environment, and how can an intelligent system learn to behave in the environment? Such a learning process should be active in this case.``With active learning, the [robotic] system may deviate from the implementation of the main application and perform actions aimed at collecting information about the environment.'' \cite{seku}.

Investigations in the field are focused on systems which \emph{act} rationally. The artificial intelligence is represented in the approach as an information processor that consumes and gives out information. The theory also tries to answer, in general, the question, ``how can a system composed of relatively unintelligent parts (say, neurons or transistors) behave intelligently?'' \cite{Cass}.

A formal description of the most intelligent agent behaviour, in the sense of some intelligence measure, is suggested in the UAI framework \cite{HuttB}, \cite{Hutt}. The framework specifies how an agent \emph{interacts} with an environment. The model is based on probabilistic modelling of the environment, and determination of the next system move, based on previous experience. Also, it is based on a numerical estimation of the system position reward, and the expected reward maximization along the trajectory. However, the method to obtain this numerical estimation is absent.

It has been demonstrated in \cite{Maximov_17}, \cite{Maximov}, \cite{UBS}, that the structure existence (a lattice structure or else a monoid structure, strictly, the linear logic structure) in the system task \cite{Maximov_17}, \cite{Maximov}, or goal \cite{UBS} set is sufficient for the system to behave quite reasonably. The behaviour even looks like that of an ant in some things \cite{Maximov_17}. Thus, we may suppose that system intelligence is the consequence of the system's purpose or predestination to use in practice. The approach does not assume environment modelling, unlike \cite{HuttB}.

In this paper, we refine the results of \cite{assa}, \cite{cog}\footnote{We give the improved notion of the payoff function in the game category used, correct proofs of Propositions \ref{th1} and \ref{th2}, and clarify the whole construction.} and develop the approaches mentioned above. We are based on the idea that it is possible to represent the study process of an environment as fulfilling of parallel achievement processes of different objects in the environment.  A tensor multiplication in a linear logic corresponds to these parallel processes. The logic is modelled in some game category \cite{resource} (first mention in \cite{abram}, \cite{abram1}). Thus, it is possible to describe the objects' achievement process by the intelligent system in the environment as a game. Position rewards in the game are represented by some sets which define the information about goals. The process of obtaining information is just the process of environment cognition.

Thus, the other problems, we investigate in this paper, are recognizing goals and the method of estimating information to obtain the game rewards. Usually, such a system has a pre-existing set of goals and a subsystem that recognizes them. However, what to do in an unknown environment? We suppose in this paper that the system behaves like a baby: it moves toward objects which have attracted its attention, chooses the most attractive item, and  looks at it from all sides. Thus, such an artificial intelligence system must be able to highlight unknown objects in the environment which were not specified in its pre-existing goal (object) set. To do this, we suggest using Biederman's theory of human image understanding \cite{bieder}. In this theory, people recognize object types by schemes that consist of geometrical primitives. Therefore, we can initially train the system's neural network to recognize these primitives only, and to build their incidence matrixes in the environment images. Then, unknown objects (combinations of primitives in the environment) with the same set of primitives and the same incidence matrix would belong to one object type. Thus, these objects may be remembered in a database.

While the system is moving, such an unknown object becomes more discernible, and the set of its primitives obtains new elements. Thus, we get an increment of information about the object. When this thing is studied from all sides, and we have no more information increment, the process of the object cognition stops. Therefore, we may take these sets of geometrical primitives as the rewards in the game which describe the system move in the environment. Similarly, while the system is moving toward a possibly known object, the uncertainty of its identification diminishes, and the set of possible identification variants may be taken as the game anti-reward (which decreases during the game in this case).

Our method of object recognition belongs to the class of local-based recognition methods \cite{tits}. However, these methods may consider any kind of local features that
described a local area of the object instead of having appearance-based parts \cite{weber}. This technique requires a data set composed of training images to select parts of the thing to be later recognized. Before training, the method cannot discriminate between multiple object classes by itself. A similar technic is used in \cite{fergus} with other methods of part selection and another learning algorithm. A more complicated approach is described in \cite{felzen2000},\cite{felzen2004}, \cite{felzen2005}, \cite{nuch}, \cite{zuffi} which uses the joint model of the parts to facilitate the detection of the features. Each part encodes local visual properties of the object, and the deformable configuration is characterized by a tree structure of connections between certain pairs of components. However, this method also requires a set of images to train part selection and to create an object class template.

On the contrary, we know features in advance in our method, and discriminate between different object classes immediately, without any training. Also, in the first steps, we do not determine these classes precisely. The information obtained is increased from step to step up to the end. After that, we get an item in the database of such classes as a set which consist of all the object schemes at once. The item represents from all sides a comprehensive image of the corresponding object at the end of the process of learning it. Subsequently, a graphical skeleton of these geon schemes could be generated, which matches the item, and a neural network can be trained to store such a 3D type as a new object as in \cite{seku1} (see Sec. \ref{sec:3.2}).


The idea of the latter image recognition approach (\cite{seku1}) is to place items of a typical shape in a different orientation in a virtual environment and to get their images. These objects are generated from some models, and the item in our database can serve as such a model. The position and the appearance of each object in each perspective are known immediately without additional calculations. In our article, we initially have  such images of basic geometric primitives only. In the process of the environment cognition, we could add more intricate object schemes in different orientations, built from primitives, as a new item in the memory. Initially, this memory is the database; afterwards, these items can be stored in the network.

The second difference of our method from the rest is that we store shape information in a static incidence matrix and use neither a probabilistic method for this purpose nor the procedure of minimizing  the energy of the constellation of parts.

The paper is organized as follows. Section \ref{sec:1} presents the backgrounds necessary for the understanding of the results. We adapt the theory of \cite{resource} to the paper purposes in Subsection \ref{subsec:2}. Section \ref{sec:3} represents a formal description of intelligent system behaviour. Here, we also refine, improve and adapt the previous results. Section \ref{sec:3.2} represents using the neural network to gain the informational rewards in the game which describes the behaviour. In Sec. \ref{rew_val} we discuss the reward valuations. In Sec.\ref{concl}, we conclude the paper.

\section{Backgrounds}
\label{sec:1}
\newtheorem{Th}{Supposition}
\subsection{Lattices}\label{subsec:1}
We suppose visible objects and, hence, the game rewards form different lattices. Thus, we start with them.
\begin{definition}
A \textbf{partially-ordered set} $P$ is the set with such a binary relation  $x \leqslant y$ on its elements, that for all $x, y, z \in P$ the next relationships hold:
\begin{itemize}
\item	$x \leqslant x$ (reflexivity);
\item	if $x \leqslant y$ and $y \leqslant x$, then $x = y$ (anti-symmetry);
\item	if $x \leqslant y$ and $y \leqslant z$, then $x \leqslant z$ (transitivity).
\end{itemize}
\end{definition}

The definition means that in the partially-ordered set, not all elements are compared with each other. This property distinguishes these sets from linearly-ordered ones, i.e., from numeric sets which are ordered by a norm.  Thus, the elements of the partially-ordered set are the objects of a more general nature than numbers. In the partially-ordered set diagram, the greater the element (i.e., vertex, node), the higher it lies, and the elements that are compared with each other lie in the same path from a bottom element to a top one. An example of a partially-ordered set diagram is represented in Fig.~\ref{fig:1} which is also a lattice diagram.
\begin{figure}
  \includegraphics[scale=0.7]{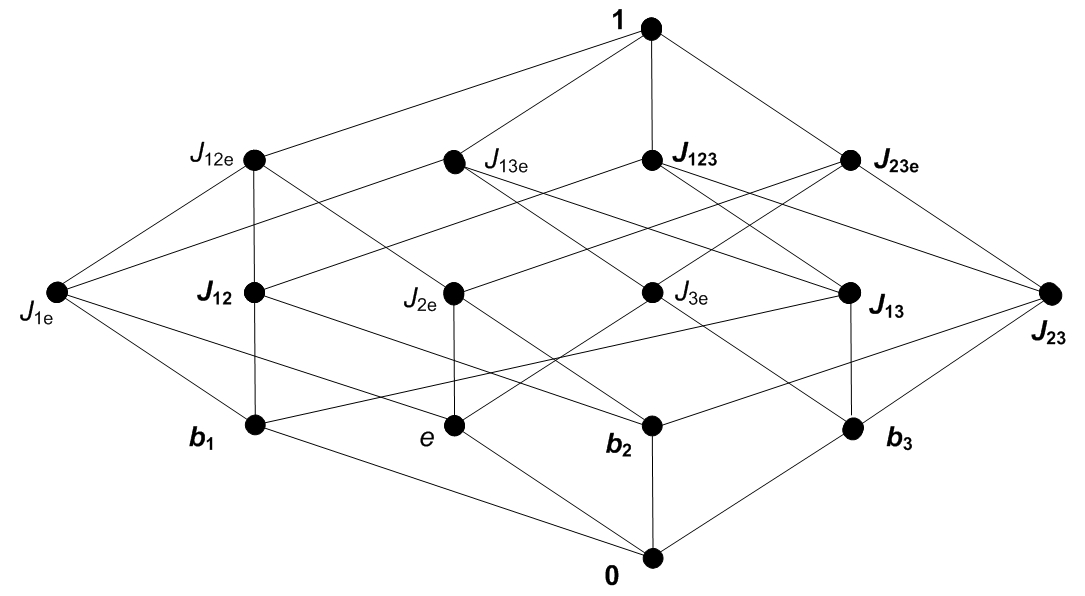}
\caption{A lattice example}
\label{fig:1}       
\end{figure}
\begin{definition}The \textbf{upper bound} of a subset $X\subset P$ in a partially-ordered set $P$ is the element $a\in P$ such that $x\leqslant a$ for all $x\in X$.
\end{definition}

The supremum or \emph{join} is the smallest subset $X$ upper  bound. The infimum or \emph{meet} defines dually as the greatest element $a\in P$ such that $a\leqslant x$ for all $x\in X$.
\begin{definition} A \textbf{lattice} is a partially-ordered set, in which every two elements have their meet, denoted by $x\wedge y$, and join, denoted by $x\vee y$.
\end{definition}

In the lattice diagram, the elements join is the nearest upper element to both of them, and the meet is the nearest lower one to both.
\begin{definition}\textbf{Generators} are such elements that generate all other elements by joins and meets (They are $b_{1}, e, b_{2}, b_{3}$ in Fig. 1).
\end{definition}

\begin{definition}The lattice is referred to as a \textbf{complete} lattice if its arbitrary subsets have the join and the meet.
\end{definition}Thus, any complete lattice has the biggest element 1, and the smallest one 0 and every finite lattice is complete \cite{Birkhoff}.

If we take such a lattice as a scale of truth values in multi-valued logic, then the largest element will correspond to complete truth (1), and the smallest to complete falsehood (0). Intermediate elements will correspond to partial truth in the same way as in fuzzy logic, partial truth is estimated by elements of the segment $[0,1]$.

In logics with such a scale of truth values, the implication can be determined through the multiplication of lattice elements (residual logics), or internally, only from lattice operations. \begin{definition}A lattice that has internal implications is called a \textbf{Brower} lattice.
\end{definition}
In such a lattice, the implication $c = a \Rightarrow b$ is defined as the largest $c: a\wedge b = a\wedge c$.
\begin{definition}The implication $\neg a = a \Rightarrow 0$ is called the \textbf{pseudo-complement}. \end{definition}
Distribution laws for union and intersection are satisfied in the Brower lattices. The converse is true only for finite lattices.

\subsection{Game Semantics}\label{subsec:2}
\begin{definition} \cite{resource}
A \textbf{Conway game} is defined as a rooted graph with vertices \emph{V} as the game positions and edges $E\subset V\times V$  as the game moves. Each edge has a \textbf{polarity} $\pm 1$  which depends on whether it is the Proponent or the Opponent move.
\end{definition}
\begin{definition} \cite{resource}
A trajectory or a \textbf{play} is some path from the graph root $\ast$. The path is \textbf{alternated} if the adjacent edges are of different polarities.
\end{definition}
\begin{definition} \cite{resource}
A \textbf{strategy} $\sigma$ of a Conway game is defined as a non-empty set of alternated plays (paths) of even length. They start from the Opponent move, closed up to the prefix of even length, i.e., for all plays \emph{s} and all moves \emph{m}, \emph{n},
$s \cdot m \cdot n \in \sigma$ implies $s \in \sigma$, and determined.  Determinism means that two different paths with a common prefix should coincide, i.e., for all plays \emph{s}, and all moves \emph{m}, \emph{n}, and \emph{n'},
$s \cdot m \cdot n \in \sigma$, and $s \cdot m \cdot n' \in \sigma$ implies $n = n'$.
\end{definition}
\begin{definition} \cite{resource}
A dual play $X^{\bot}$ is obtained from the play \emph{X} by reversing the
polarity of moves.
\end{definition}
\begin{definition} \cite{resource}
The tensor product $X\otimes Y$ of two Conway games \emph{X} and \emph{Y} is the product of the two underlying graphs, i.e., positions $x\otimes y$ are $V_{X\otimes Y}=V_{X}\times V_{Y}$ with the root $\ast_{X\otimes Y}=\ast_{X}\times \ast_{Y}$, moves are $x\otimes y\rightarrow \left\{\begin{aligned}z\otimes y; x\rightarrow z \;in \;X\\
                      x\otimes z; y\rightarrow z \;in \;Y\end{aligned}
\right.$ and the polarity of a move in $X\otimes Y$ is inherited from the polarity of the
underlying move in \emph{X} or \emph{Y}.
\end{definition}
Generalized linear logic is modelled in the category \textbf{Conw} of such games \cite{resource}. The category objects are Conway games, and morphisms $X\rightarrow Y$  are strategies in  $X^{\bot}\parr Y$. The definition of the categorical construction of the operation $\parr$, which is dual to tensor $\otimes$, is not discussed here for simplicity because this is not essential for our description. It is enough to mention that on game graphs, these two operations are the same, so in \cite{resource}, they are not even distinguished. The morphism composition and identity morphism are apparent \cite{resource}. We do not need the construction of the linear logic here, except the notion of linear implication.
\begin{definition} \cite{resource}, \cite{llp}
The linear implications $X\multimap Y$ in the category are defined as

$X\multimap Y = X^{\bot}\parr Y$
\end{definition}
\noindent since the category is symmetric monoidal closed (thus, we may define linear logic and the implication in the category).

A Conway game \emph{X} with a payoff is the game with an additional weight $k_{X} = \{1, 1/2, 0\}$ in each vertex \cite{resource}.  The weight depends on whether the position is winning or not. In the tensor production and implication, these weights obey the usual rules of conjunction and implication for such truth values scale \cite{resource}. Thus, the Conway's payoff game $X\otimes Y$ is defined as the underlying Conway game $X\otimes Y$, equipped with the payoff function $k_{X\otimes Y}(x\otimes y) = k_{X}(x)\wedge k_{Y}(y)$. The Conway's payoff game $X\multimap Y$ is defined as the underlying Conway game $X\multimap Y$, equipped with the payoff function $k_{X\multimap Y}(x \multimap y) = k_{X}(x)\Rightarrow k_{Y}(y)$. A strategy $\sigma$ on a Conway's payoff game \emph{X} is winning when every play $s : x \mapsto y$ in the strategy ends in a winning position $y$, i.e., in a position of payoff 1/2 or 1.

It is possible to prove that the categorical construction is conserved if the weights' numbers are replaced with some sets which form a Brouwer lattice. Also, the operations of \cite{resource} must be replaced with lattice operations. Thus, we may declare by definition
\begin{equation}\label{eq1}k_{X\otimes Y}(x\otimes y) = k_{X}(x)\wedge k_{Y}(y)
\end{equation}with lattice $\wedge$,
\begin{equation}\label{eq2}k_{X^{\bot}\parr Y}(x^{\bot}\parr y) = \neg k_{X}(x)\vee k_{Y}(y)
\end{equation}with lattice $\vee$, and
\begin{equation}\label{eq3}k_{X\multimap Y}(x \multimap y) = k_{X}(x)\Rightarrow k_{Y}(y)
\end{equation}where $\Rightarrow$ is now the lattice implication. Both of the last two definitions are suitable as the payoff function of the game $X\multimap Y$. The first definition is a particular case of the second one in Boolean lattices, and it may be used in practice as more convenient.

However, we must emphasize that these payoff definitions are purely voluntaristic. They do not follow from any mathematical construction. The unique reason to introduce them is interpreting the tensor production $\otimes$ as a multiplicative conjunction and the co-tensor production $\parr$ as multiplicative disjunction in linear logic. Hence, we may invert the payoff definitions  \ref{eq1} and \ref{eq2}, since operations $\otimes$ and $\parr$ are mutually dual as $\wedge$ and $\vee$. Thus, we introduce:
\begin{equation}\label{eq1n}k_{X\otimes Y}(x\otimes y) = k_{X}(x)\vee k_{Y}(y)
\end{equation}
\begin{equation}\label{eq2n}k_{X^{\bot}\parr Y}(x^{\bot}\parr y) = \neg k_{X}(x)\wedge k_{Y}(y).
\end{equation}
We will use these definitions in Sec. \ref{sec:3} since we interpret in the section the fulfilling of parallel processes in the tensor production as the processes' join in the corresponding lattice. We will prove that the categorical construction is also conserved for these definitions.

The larger set is connected with a position, the more advantageous it is. We suppose the existence of a universal set containing all the others. Thus, all such estimation sets form a complete lattice.
\begin{definition}
We call a strategy $\sigma$ on a Conway's payoff game \emph{X} with position estimations in a lattice as \emph{winning} if every play $s : x \mapsto y$ in the strategy ends in a position $y$ of payoff in the lattice which is different from 0.
\end{definition}
Let us note here, that linear implication $X \multimap Z$ means the process (game) $X$ consuming and the process (game) $Z$ obtaining \cite{jerar1987}. Therefore, the payoff $k_{Z}(z)$ may not be 0 in the end position of the winning strategy $\sigma$ of the game $X \multimap Z$, because we must not obtain a process (game) $Z$ ended in 0 in the winning case. However, this restriction is unnecessary in the case of definitions (\ref{eq1n}) and (\ref{eq2n}). We show now that such defined winning strategies do compose as in \cite{resource}.
\begin{proposition}\label{th1}
The strategy $\rho\circ\sigma : X \multimap Z$  is winning when the two strategies $\sigma : X \multimap Y$ and $\rho : Y \multimap Z$ are winning.
\end{proposition}
\begin{proof}
(for (\ref{eq3})): It is known that strategies do compose \cite{resource}. Thus, it is sufficient to check the winning condition. We should observe that the composition of two winning positions is winning:

$\begin{aligned}k_{X}(x)\Rightarrow k_{Y}(y)>0,\; and\; k_{Y}(y)>0, \;i.e.,\; x\multimap y \;is\; winning;\\
                     k_{Y}(y)\Rightarrow k_{Z}(z)>0,\; and \; k_{Z}(z)>0, \;i.e.,\; y\multimap z \;is \;winning;\\
                     implies\; k_{X}(x)\Rightarrow k_{Z}(z)>0, \;i.e.,\; x\multimap z \;is\; winning.\end{aligned}$

\noindent However, it is evident, since $k_{X}(x)\Rightarrow k_{Z}(z)\geqslant k_{Z}(z)>0$.




(for (\ref{eq2})): It is evident that in the case of $k_{X\multimap Y}(x \multimap y) \equiv k_{X^{\bot}\parr Y}(x^{\bot}\parr y) = \neg k_{X}(x)\vee k_{Y}(y)$ and $k_{Y}(y) > 0$ winning strategies do compose.

(for (\ref{eq2n})): It is evident that in the case of the winning strategy with $k_{X\multimap Y}(x \multimap y) \equiv k_{X^{\bot}\parr Y}(x^{\bot}\parr y) = \neg k_{X}(x)\wedge k_{Y}(y)$, $\neg k_{X}(x) > 0$ and $k_{Y}(y) > 0$, thus winning strategies do compose.
\end{proof}
\begin{definition}
Let us, \textbf{SetPayoff} is a category whose objects are Conway's payoff games, in which position weights take values in a lattice, and morphisms $X \rightarrow Y$ are winning strategies in $X \multimap Y$.
\end{definition}
\begin{proposition}\label{th2}
The category \textbf{SetPayoff} is symmetric monoidal closed.
\end{proposition}
\begin{proof}
(for (\ref{eq1}) and (\ref{eq3})): The category of Conway games is symmetric monoidal close \cite{resource}. Therefore, it is sufficient to check if $(k_{X}(x)\wedge k_{Y}(y))\Rightarrow k_{Z}(z) = k_{X}(x)\Rightarrow (k_{Y}(y)\Rightarrow k_{Z}(z))$ for all positions. But this formula is valid in Heyting algebras (i.e., in Brouwer lattices).

(for (\ref{eq1}) and (\ref{eq2})): $\neg(k_{X}(x)\wedge k_{Y}(y))\vee k_{Z}(z) = \neg k_{X}(x)\vee (\neg(k_{Y}(y)\vee k_{Z}(z))$ in Brouwer lattices.

(for (\ref{eq1n}) and (\ref{eq2n})): $\neg(k_{X}(x)\vee k_{Y}(y))\wedge k_{Z}(z) = \neg k_{X}(x)\wedge (\neg(k_{Y}(y)\wedge k_{Z}(z))$ in Brouwer lattices.
\end{proof}

Thus, the symmetric monoidal closed categorical construction for Conway's payoff games from \cite{resource}, is conserved for lattice payoffs.

\subsection{Biederman Geons}\label{subsec:3}

Irving Biederman has proposed in \cite{bieder} a theory of human image understanding in which an image is segmented into a set of geometric primitives, such as blocks, cylinders, wedges, and cones. The collection of the components, called geons, is rather limited ($N\leqslant 36$) and can be derived from easily detectable properties of edges: curvature, collinearity, parallelism, and convergence. The detection of these properties is invariant over viewing position and image quality, and thus allows perception in different object positions and in the case of noised image. The experiments of \cite{bieder} showed low errors in object naming of such geon schemes of different objects. Some of such simple geon schemes are represented in Fig. \ref{fig:1.5}:
\begin{figure}\begin{center}
  \includegraphics[scale=0.4]{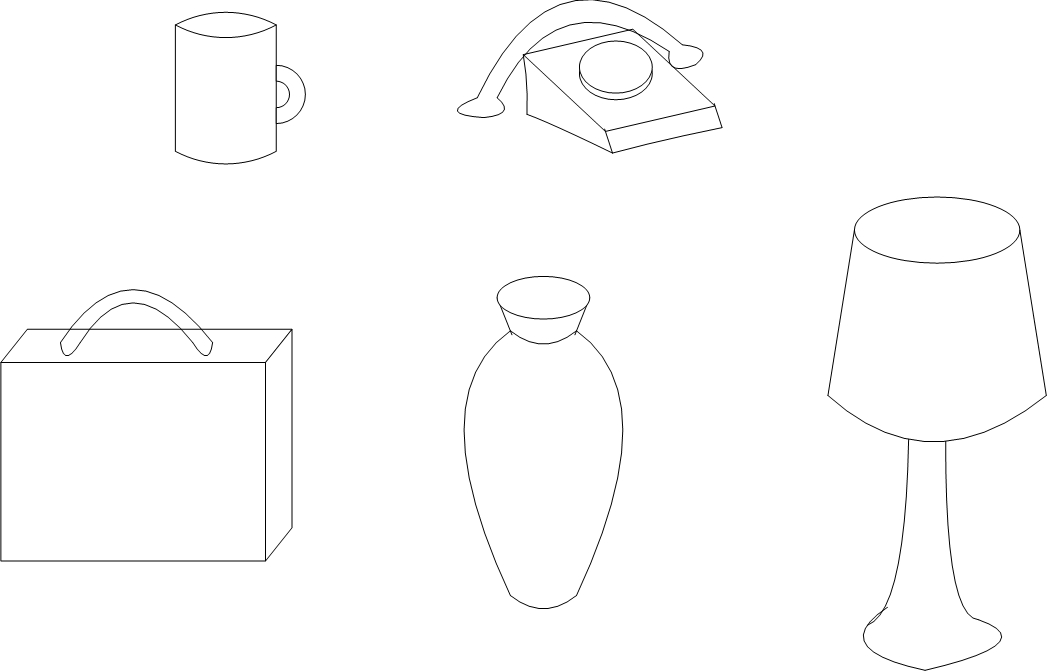}
\end{center}
\caption{Examples of geon schemes of simple objects.}
\label{fig:1.5}       
\end{figure}
The approach is similar to speech perception: in both cases, we can code tens of thousands of objects by mapping the input onto a limited number of primitives and then using a representative system to build free combinations of these primitives. However, this method does not allow us to distinguish images inside one object type. To do this, we should use other methods.

\section{System Behaviour Description}\label{sec:3}
We consider the cognition process as a game in which a system investigates an environment, i.e., the system obtains information about the environment objects in the form of geon sets.
\begin{Th}It is supposed that the system investigates the environment visible up to some horizon in each direction  (as if in a fog: the objects visible best are nearer to the system, Fig.~\ref{fig:2}) and builds images of the observed objects.
\end{Th}
\begin{figure}\begin{center}
  \includegraphics[scale=0.7]{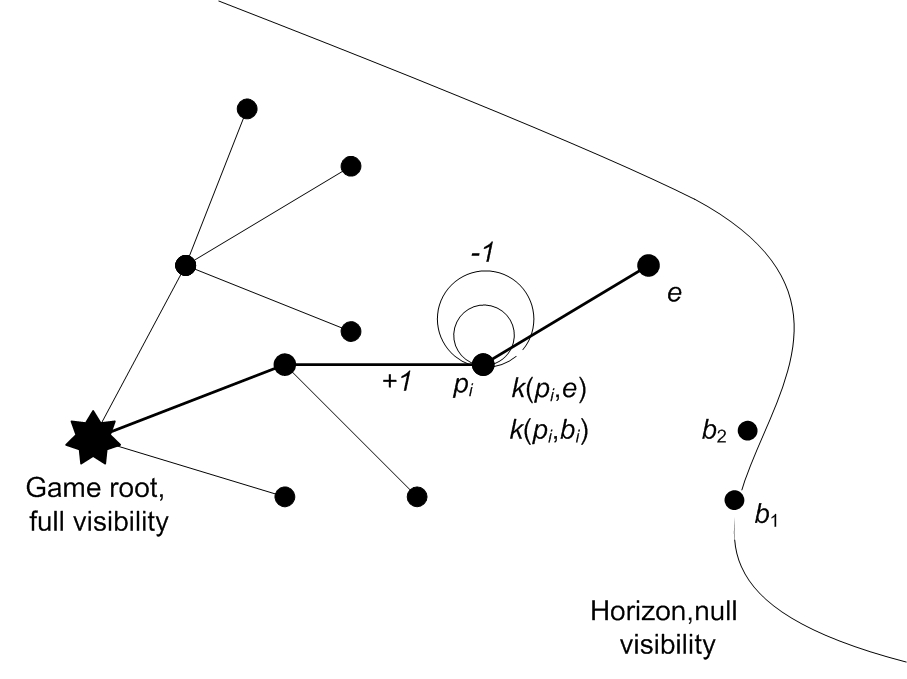}
\end{center}
\caption{An example of a game and its visibility horizon. The bold line shows the resulting play with the rewards $k(p_{i},b_{j})$ in the position $p_{i}$ of the goal $b_{j}$ achieving process.
The environment moves are depicted by circles.}
\label{fig:2}       
\end{figure}

\begin{Th}It is supposed that objects' attractive degrees guide the system to behave: it investigates things in the environment which have attracted its attention, and builds their images.
\end{Th}
The theory of attention is discussed, e.g., in  \cite{Cog_Sci}. ``Attention is defined as a concentrated mental activity. In general, we can think of attention as a form of mental activity or energy that is distributed among alternative information sources''\cite{Cog_Sci}. Both general classes of theories that attempt to explain attention --- bottleneck theories and capacity theories --- explain \emph{how} attention selects information sources. Still, they tell us nothing about \emph{why} exactly it makes the selection. We suppose the system has some pre-existing preferences to choose the sources. In the simplest cases, the choice can be based on the detection of areas of difference in hue, contrast, degree of distance from the observation point, etc. \cite{seku} In more complicated cases, the system can have some frames or gestalt images. Thus, we suppose the system predestination or purpose for practical use determines the attention preferences.

Then, the system should distribute the objects by the attraction degree during the investigation. Thus, some objects are more attractive, some of them are less attractive, and some cannot be compared by the attraction degree. Therefore, we get a partial order at the object set. It is supposed that the set has the bottom element, i.e., the element of null interest and the top element, i.e., the object of the greatest interest, which is the most attractive element. The latter one is the lattice join of all its elements (see Supposition \ref{sup3}).
The other elements may be represented as some joins (combinations) or meets in the case when an attractive object is some part of another thing. Therefore, the system builds a complete lattice of the environment object images, their combinations, and allotments (e.g., Fig. \ref{fig:1} for objects from Fig. \ref{fig:2}). Such a lattice may just exist; thus, the attractions of new objects may be combined with the attractions of pre-existing ones. Many simple biological systems, e.g., ants, possibly have such a preconstructed lattice of objects they can have a deal. Perhaps the partial order in the lattice may be dynamically changed for more complicated systems.

In any case, we suppose the lattice of the environment object images may be dynamically changed due to the system not seeing everything around at each moment. The system looks into a sector and builds images that lie inside it.


\begin{Th}\label{sup3}It is supposed that the preferable behaviour of the system is to achieve all its goals at a given time, i.e., to get maximum information.
\end{Th}
We identify the goals and their achievement processes. Thus, the most preferable behaviour variant corresponds to the top lattice element 1. And the bottom element 0 corresponds to complete inactivity and to the least significant behaviour variant. All the estimations may be considered as partially true truth values. Thus, we can say that the more essential the behaviour is, the truer it is. When it is not possible to determine the most crucial estimation for a behaviour due to the order being only partial, some additional methods may be used to select the optimal variant \cite{assa}.


The process of the system moves in the environment space, i.e., the system's cognition process, can be represented as a game in which the environment informs the system about (partially) visible objects at each step with the position reward that estimates the information. In the next paragraph, we will see that the rewards are sets of geometric primitives which make up the objects.

We regard the environment as the Opponent and the system as the Proponent to use the categorical construction of Sec.~\ref{subsec:3}. The Proponent moves from one position in the environment to the other by the use of the information to achieve his goals, i.e., environment objects. The more fortunate the position, i.e., the larger the reward, the more precise information the environment provides about the item in the position. Thus, the completely winning position is the last game step of the Proponent, in which it can get no more information. The system is placed initially in the configuration space (environment) in the root * of the system game \emph{A} with the system goal/object lattice $M_{s}$ partially ordered according to system attention preferences.

Game \emph{A} represents the \emph{\textbf{possible}} system's moves in the environment. But the \emph{\textbf{real}} trajectory or the play is chosen from the demand of the maximal total position reward along the projected path. The system move in the environment is estimated corresponding to such a criterion with the reward $k(p_{i},b_{j})$ in the position $p_{i}$ of the goal $b_{j}$ achieving process.

It may be that the system does not intend to achieve any goal initially and moves according to a criterion of an optimal move in a goals' absence. Thus, the system has a goal (task) $a$ of a free movement in the environment, which should be included in the lattice $M_{s}$. Hence, such a free movement should have its value. We do not discuss here the optimality criterion, since the robot or agent systems' designers may suggest such ones according to their needs in a specific environment. However, we suppose the free movement rewards in the corresponding game are also valuated by some sets. These sets are of the same nature as the sets which correspond to the rewards in the game of achieving goals. Therefore, the rewards of the free movement are some sets of geometric primitives that are placed in the environment.

Thus, we suppose the system sees different combinations of geometric primitives in the environment. The primitives which are located in one place make up an object scheme. These schemes may have joins and meets combining other such schemes (which may no longer be in one place). Hence, such primitives' sets form a lattice. The lattice is complete since it contains any joins and meets. This lattice is different from the lattice $M_{s}$ of the system goals. The lattice $M_{s}$ contains only a part of all objects, and it is partially ordered according to attention preferences. The primitives' lattice contains all visible things and their combinations, and it is ordered according to the relation of the set inclusion. We demand the primitives' lattice is Brouwer (in particular, it must be distributive); thus, it includes all that is necessary for these joins and meets even when they are not evident.

Let us \emph{n} goals $b_{1} ... b_{n}$ are discovered in the environment by the system with information about them  $k(p_{i},b_{j})$ in positions $p_{i}$ of the game \emph{A} (Fig. \ref{fig:2}). The rewards $k(p_{i},b_{j})$ take values in the primitives' lattice. Only $k$ goals from these $n$'s ones may be chosen due to limited system opportunities. Game \emph{A} corresponds to the process of achieving the free movement goal \emph{a}. Then, a winning strategy of the game $A'=A\multimap B_{1} \otimes ... \otimes B_{k}$ defines a transition (morphism) from the game \emph{A} to this new game $A'$ of moving and achieving goals $b_{1} ... b_{k}$  in the games $B_{1}  ...  B_{k}$. It is supposed that the system can achieve several goals in parallel up to the moment when only one object rests to be chosen. Thus, the game $A'$ corresponds to parallel processes of achieving those \emph{k} goals from discovered \emph{n}'s ones, which may be better achieved in the next sense.

It is reasonable to choose the trajectory (play) from the demand to maximize the reward along the path within the visibility horizon in the following way (using (\ref{eq1n}), (\ref{eq2n}) in contrast to \cite{assa}, \cite{ubs19})\footnote{We may also use formulas (\ref{eq3}) and (\ref{eq1}) for rewards: \begin{equation}k^{A\multimap B_{1} \otimes ... \otimes B_{k}}_{play} = \max_{plays}\bigcup_{play}[k^{A}\Rightarrow (k^{B_{1}}\wedge ... \wedge k^{B_{k}})]
\end{equation} But this form may be inconvenient in practice. However, the implication calculation is no more difficult \cite{maximov20}, hence, the form may also be used.}:
\begin{multline}\label{eqmain}
k^{A\multimap B_{1} \otimes ... \otimes B_{k}}_{play}(a\multimap b_{1}\otimes ... \otimes b_{k}) \equiv  k^{A^{\bot}\parr B_{1} \otimes ... \otimes B_{k}}_{play} = \\ = \max_{plays}[\bigcup_{play}\neg k^{A}(a)\wedge(k^{B_{1}}(b_{1})\vee ... \vee k^{B_{k}}(b_{k}))]
\end{multline}
Here, the reward $k^{A\multimap B_{1} \otimes ... \otimes B_{k}}_{play} = k^{A^{\bot}\parr B_{1} \otimes ... \otimes B_{k}}_{play}$ is maximized in the game $A'$ and corresponds to that process of achieving \emph{k} goals that has the greatest priority (i.e., the highest truth value)
in the system goal lattice (at the current time). Thus, these \emph{k} processes are the most important parallel ones from the viewpoint of the system goal lattice. The priority is maximal among all possible parallel processes of achieving \emph{n} discovered goals\footnote{We consider here the join of several goals in the lattice $M_{s}$ as the process of their possibly parallel achieving (like in \cite{Maximov_17}), though, in \cite{assa} a linear logic structure on the goal lattice was used in this case. We did so because it is hard making sense of lattice elements multiplication in linear logic.}.

The maximum in (\ref{eqmain}) is taken among all possible plays, and it joins the rewards along these plays (i.e., trajectories) in games $A^{\bot}$, $B_{1}$, ... and $B_{k}$. Thus, information about all objects $b_{i}$ and their system images is demanded to maximize along the resultant path up to the moment when all the images together may not be able to improve. Then, the total number of chosen objects should be decreased by the same method: the system should pick those $l$ objects with $l < k$, which parallel achieving processes have the largest estimation in the system goal lattice. And so on, up to the moment, when perhaps only one goal remains. However, some external method of the path calculation may suggest an objects' traversal path to obtain the maximum cumulative reward from as many angles as possible (see Sec. ``System Movement Discussion  and Reward Valuations'' \ref{rew_val}).

The meaning of formula (\ref{eqmain}) is that, following the semantics of linear implication, the system moves from the execution of the process of free movement $A$ to the processes of achieving goals $B_{j}$. The gain is calculated for the strategy (i.e., the path) of game $A^{\bot}\parr B$ (with simplified notation). At the same time, information about targets $k^{B_{j}}(b_{j})$ is maximized, and information about the free movement of $k^{A}(a)$ is minimized (since pseudo-complements $\neg k^{A}(a)$ are maximized). Such $k$ and $\neg k$ can be considered as arguments and counterarguments for the corresponding movement following the ideas of the JSM method of plausible inference \cite{jsm}: the more information we have about the goal, the stronger the arguments for the transition to achieving it. Also, the stronger the arguments against free movement (i.e., the less information we have about the possibility of such a movement), the stronger the arguments for moving from free movement to achieving a goal.

\section{Object Recognition}\label{sec:3.2}
Thus, we come to the main problem: how to select (recognize) objects in the environment and how to evaluate the amount of information about them, i.e., the rewards? The latter problem we solve in Sec. \ref{rew_val} ``System Movement Discussion  and Reward Valuations'', and the former one we discuss below.

We suggest decomposing an environment image to Biederman's geons (paragraph ``Biederman Geons'' \ref{subsec:3} in Sec. \ref{sec:1}). Then, during the cognition process, the system obtains such images from different camera angles, and gets a tuple of sets of geons localized in one place. These sets with their geon incidence matrixes constitute one type of environment objects. The incidence matrix for each geon pair indicates their facets which intersect (see paragraph ``Geon Recognition'' \ref{recog} below). If the system recognizes such a set in future, it will know the possible object type.

To perform the programme, we need automatic training of the neural network to be used for recognition, initially to these geons only, and, afterwards, to these geon sets.
Thus, we consider first an automatic generation of training samples in a virtual environment \cite{seku1}.

\subsection{Generation of training samples}
The presence of a high-quality training set largely determines the efficiency of machine learning algorithms. It is worth mentioning that in the preparation of the training samples attention should be paid, not only to the volume of data, but also to such things as the balance of classes and the order of their sequence. The data must contain a comparable amount of instances for each class, and must be mixed. It is advisable to include such data in the training set that is as close as possible to the conditions of further use of the neural network.

In this study, the technology of synthesis of training sets is proposed based on the use of three-dimensional graphics (OpenGL library \cite{gl}) within the developed software-algorithmic complex (Fig. \ref{fig:4}).
\begin{figure}\begin{center}
  \includegraphics[scale=0.4]{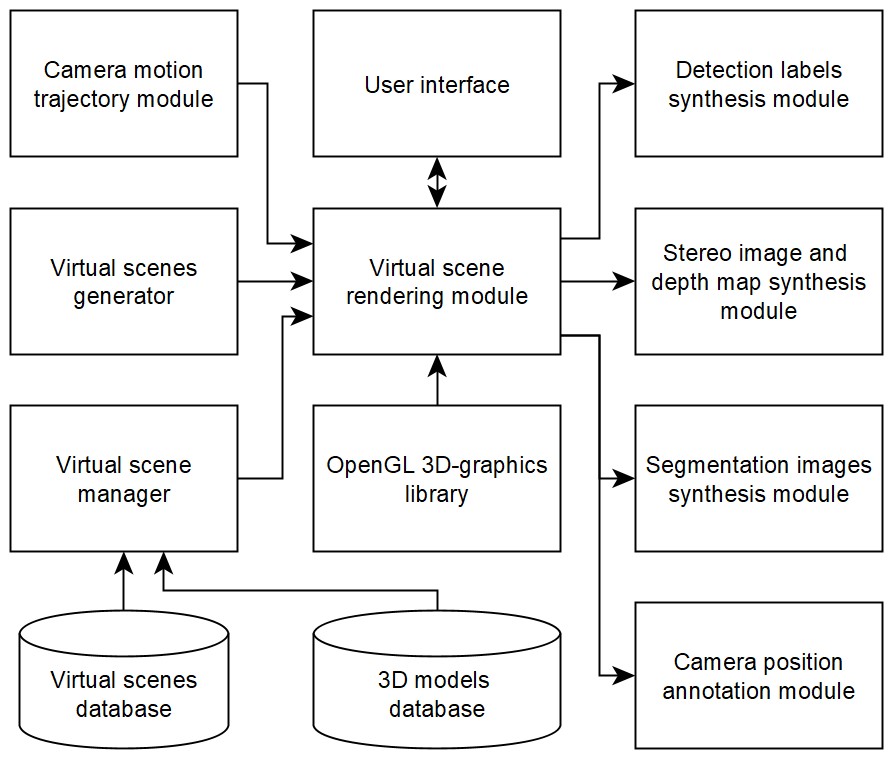}
\end{center}
\caption{The structure of the software for generation of training sets.}
\label{fig:4}       
\end{figure}

The dataset synthesis software (DSS) is based on the module for the generation of images of virtual scenes. A virtual scene is a collection of three-dimensional objects of different categories, equipped with a description of their position, orientation and colour characteristics.

\subsection{Training sets formation in problems of visual classification}
In accordance with tasks of visual image analysis, DSS allows the generation of training samples for solving the problems of visual classification, localization, segmentation and image depth evaluation. In addition, the virtual environment provides access to the exact position of the camera at successive times, which makes it possible to synthesize training samples to solve the problem of visual odometry.

There are two ways to create virtual scenes in the pre-rendering phase.

For the tasks of coarse tuning of neural network classifiers, when the mutual position of different objects is not important and, to the contrary, requires the greatest possible variety of objects moving around the scene, the approach is applied, the essence of which is as follows. The number of N objects simultaneously observed in the scene is manually set or randomly selected. A list of random positions for these objects is formed:
\begin{equation*}
P = \{p_{1}, ..., p_{N}\}.
\end{equation*}
The situations of mutual penetration of objects are eliminated on the basis of the method of potential fields:
\begin{equation*}
p'_{i} = p_{i} = min(d_{max}, \sum_{j=1,j\neq i}^{N}\eta/(p_{j} - p_{i})^{2},
\end{equation*}
where $p'_{i} = \{x', y', z'\}$  is an updated object position; $d_{max}$  is a maximum shift; $\eta$ is the objects attraction coefficient.

For the tasks of configuring neural network classifiers to solve specific application problems, an approach based on the loading of a pre-prepared virtual scenes is used. In such a case, the variety of training examples is achieved, not by changing the position of objects in the scene, but by changing the angle of observation during the movement of the camera along the specified trajectory.
At the first stage of solving a specific application problem for setting up a visual classifier, software is used that allows the formation of descriptions of virtual scenes in the following form:
\begin{equation*}
W = \{o_{1}, ..., o_{N}\};
\end{equation*}
where $o_{i}$  is a programme structure, encapsulating the object's position, orientation, type and specifics of its visual appearance. Compatibility of virtual scenes storage format is ensured to provide their correct loading in the virtual environment.

In the second stage, automatically generated scenes are loaded into the DSS: the text descriptions of the scenes are interpreted and the corresponding software representations for the objects listed in the scene file are formed. Next, the reliable results of visual analysis are determined on the basis of a direct access to the properties of the loaded programme structures. Annotations containing the desired results of the analysis of the type and position of objects are saved together with the images in a directory on the disk for further configuration of neural networks.
Note that the first stage can be performed either with the use of ready-made three-dimensional models as the elements of the set, or with the use of models synthesized procedurally. The latter approach is preferable due to the unlimited possibilities for altering the parameters and as a result achieving variability in the appearance of these objects.

\subsection{Procedural model generation}
Procedurally generated objects are specified as a set of parameters, which, on the one hand, is more compact compared to the explicit enumeration of constituting geometric primitives in the three-dimensional model file, and on the other hand, allows the appearance of the object to diversify for the formation of training sets of high quality.
In general, the generated object is specified by the formula
\begin{equation}\label{eqsec3}
o = \{l, p_{1}..., p_{K}\},
\end{equation}
where $l$  is object class, $p_{i}$  are parameters are determined by the expert which affect the geometric shape of the object through the relations embedded in the corresponding analytical model.
In Fig. \ref{fig:5} examples of procedurally synthesized objects for modelling both indoor and outdoor areas are presented.
\begin{figure}\begin{center}
  \includegraphics[scale=1.0]{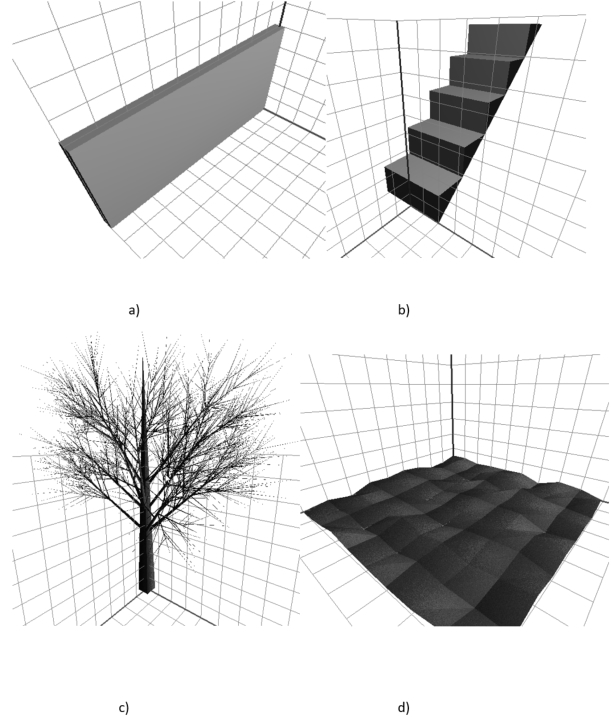}
\end{center}
\caption{Examples of automatically generated objects:
a) wall; b) stairs; c) tree; d) a fragment of the landscape.}
\label{fig:5}       
\end{figure}
For example, to generate an object of type "Tree", the set (\ref{eqsec3}) takes the form:
\begin{equation*}
o = \{``Tree'', h, k, n_{b}, n_{l}\},
\end{equation*}
where $h$  is the tree trunk height, $k$  is a thickness coefficient, $n_{b}$  is a number of branches at every level of the tree, $n_{l}$  is a number of levels of recursive branching.

Each branch is formed as a typical four-sided pyramid, the displacement and rotation of which are set randomly within the limits allowed by the parent object (trunk or branch), and the length is selected in proportion to the distance from the positioning point of the branch to the final vertex of the parent object. This approach allows a wide class of objects similar to real trees of different types to generate. The same can be said about the other objects presented in Fig. \ref{fig:5}.

The described possibility of procedural generation of objects can be used not only for the formation of training sets in the classification of visual images, but also when setting up vision systems that solve the problem of avoiding obstacles. The presence of comprehensive information about the geometry of the synthesized object enables optimal trajectory planning in the nearby space. Matching the resulting motion plan with the partial visual information available for the moving camera allows the formation of a dataset with samples capturing implicit relationships in choosing adequate robot motion for different observed situations in the external environment.

An additional increase in the number of training examples can be obtained using well-known methods of augmentation: shifting, rotating, reflecting, scaling, noising, blurring images, as well as a relatively new method of neural network data augmentation \cite{aug}.

\subsection{Geon Recognition}\label{recog}
Now, we can flesh out the formula (\ref{eqsec3}) for geons:
\begin{equation}
geon_{i} = \{id_{i}, \{(facet_{j}, line_{j}, deformed, curved)\}\},
\end{equation}
where ``line'' is a linear parameter of the $j$'s facet and takes values $\{long, short\}$, ``curved'' is its curvature parameter and takes values $\{yes, no\}$,  and ``deformed'' parameter takes values $\{no, bloat, depressed\}$. The facet set is:\\ $\{top, bottom, front, back, left, right\}$ with the corresponding geometrical form for each value. All geons are modifications of a cylinder and a brick. Accordingly, each brick is a combination of rectangles (and trapeziums up to triangles) that have a long and a short size. The cylinder also has such an aspect ratio. The ``line'' parameter indicates this characteristic. Additionally, geons may be parameterized with scaling factors, e.g., $s_{x}=[0.5 ... 2], s_{y}=[0.5 ... 2], s_{z}=[0.5 ... 2]$, and surface noise $F = [0...1]$, which allow us to diversify their visual appearance during the generation of the training images. We have used such a diversification in this paper.

Then, we get the following values for the incidence matrix of 
a geon pair:
\begin{equation}
M_{i,j} = \left\{\begin{smallmatrix}\{zone_{j}\},\; geon_{i}\bigcap geon_{j} \neq \emptyset;\\
          0,\; geon_{i}\bigcap geon_{j} = \emptyset,
          \end{smallmatrix}
          \right.
\end{equation}
where the variable ``$zone_{j}$'' refers to $j$'s geon and takes the next values:
\begin{equation}
zone_{j} \in \{top, bottom, left, right, front, back\}_{j}.
\end{equation}
The set of zones denotes those facets in $j$'s geon which intersect with $i$'s geon. We get corresponding $i$'s geon facets in the $M_{j,i}$ element. Thus, we can reconstruct the whole geon scheme of the object in a certain foreshortening by its incidence matrix.

The incidence matrix and zone values may be obtained with fuzzy operations (especially, when the objects intersect by two or more zones) from the objects' bounding box coordinates which are given by the YOLO neural network used.
Naturally, this is rather a rough object classification, and we use it as the first simple variant.

At the first stage, these matrices can be storage in the system's (robot's) memory in order to compare them with the matrices of observed objects. Subsequently, we would like to build a 3D object by a corresponding incident matrix set and to train a network for such objects as we do in this paper for geons. Such a 3D object could be supplemented by a graph of geon connections to each other.

In this paper, as a first step, we have used the method of previous sections, obtained a geon data set\footnote{We use a set of 16 main geons for our limited calculation capabilities.}, and trained a YOLO neural network on them with the resources \cite{nels}, \cite{nelsj}. At the network input, we have a raster RGB image, presented as an array of pixel brightness with a resolution of 416 * 416.
At the output of the network, we obtain 169 cells, each of which contains $N$ values:
$N =$ (4 coordinates of the frame $+ 1$ estimate of the presence of an object $+ 16$ geon classes) * 3 proportions of the frames $= 63$ values.
Total: $169 * 63 = 10647$ values.

Fig. \ref{fig:6} --- Fig. \ref{fig:11} show examples of an image decomposing on geons after the network training.

\begin{figure}\begin{center}
  \includegraphics[scale=0.6]{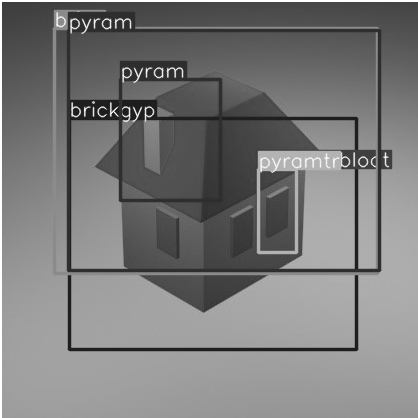}
  \includegraphics[scale=0.6]{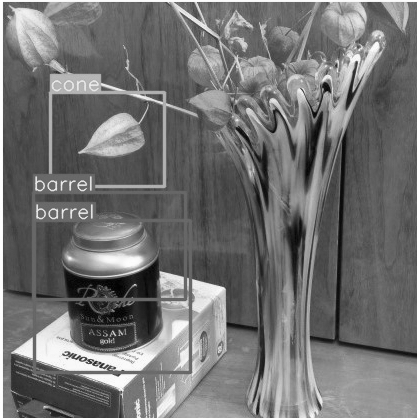}
  \includegraphics[scale=0.6]{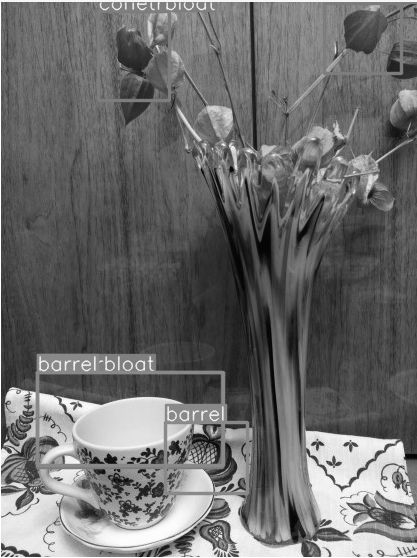}
\end{center}
\caption{An example of a simple image decomposition in the case of automatically generated
 images in the training set.}
\label{fig:6}       
\end{figure}
\begin{figure}\begin{center}
  \includegraphics[scale=0.6]{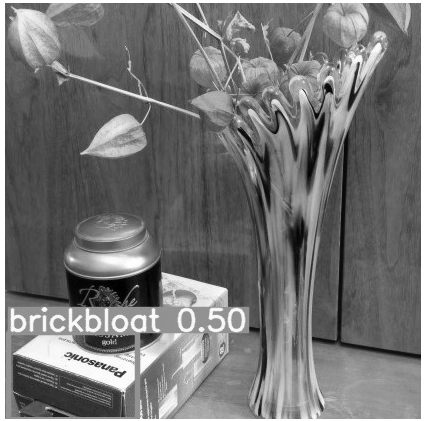}
  \includegraphics[scale=0.6]{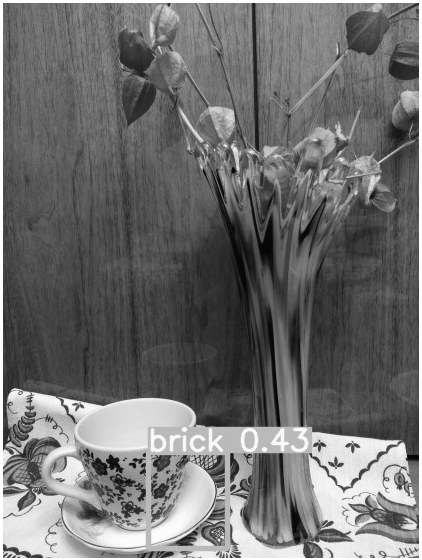}
\end{center}
\caption{An example of a simple image decomposition in the case of
manually marked up images in the training set.}
\label{fig:7}       
\end{figure}
\begin{figure}\begin{center}
  \includegraphics[scale=0.6]{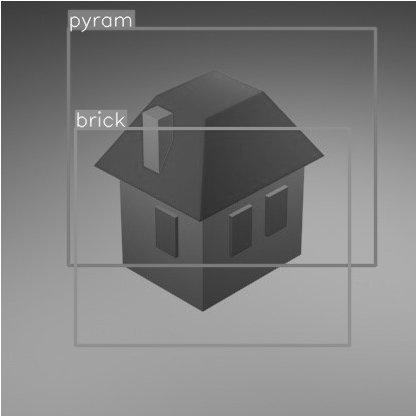}
  \includegraphics[scale=0.6]{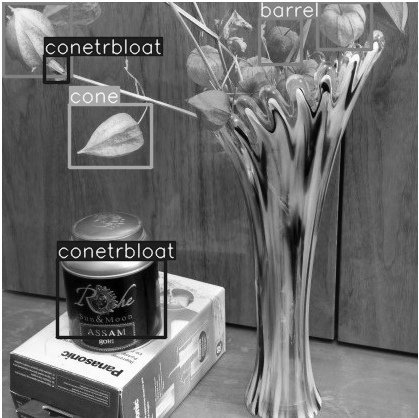}
  \includegraphics[scale=0.6]{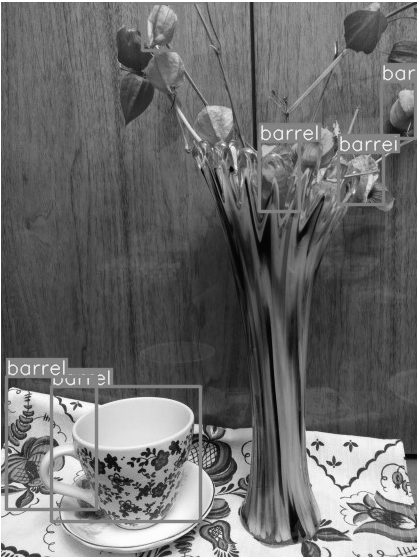}
\end{center}
\caption{An example of a simple image decomposition in the case of automatically generated
 and manually marked up images in the training set.}
\label{fig:8}       
\end{figure}
\begin{figure}\begin{center}
  \includegraphics[scale=0.6]{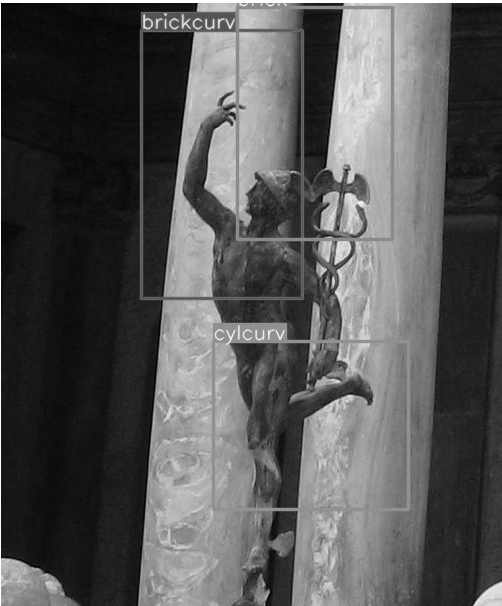}
  \includegraphics[scale=0.7]{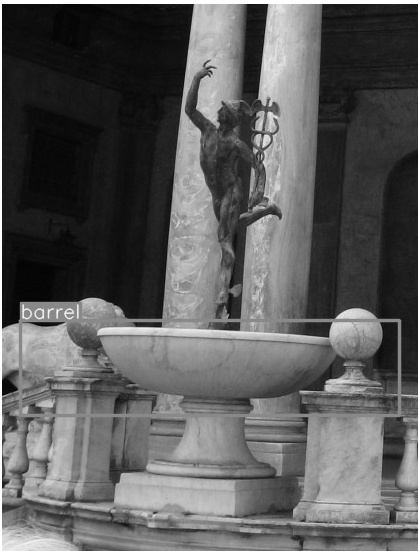}
  \includegraphics[scale=0.7]{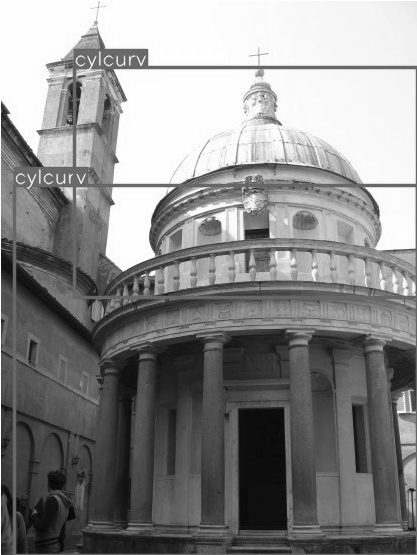}
  \includegraphics[scale=0.7]{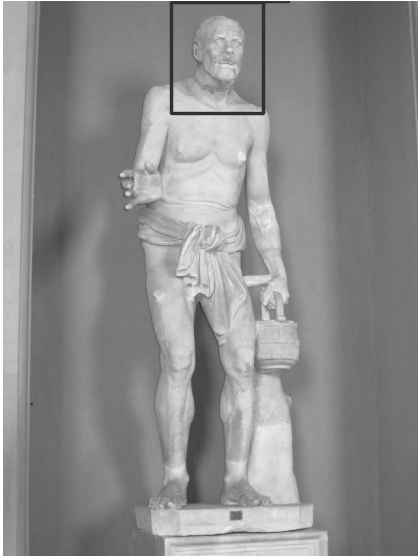}
\end{center}
\caption{An example of a complicated image decomposition in the case of automatically generated
 images in the training set.}
\label{fig:9}       
\end{figure}
\begin{figure}\begin{center}
  \includegraphics[scale=0.5]{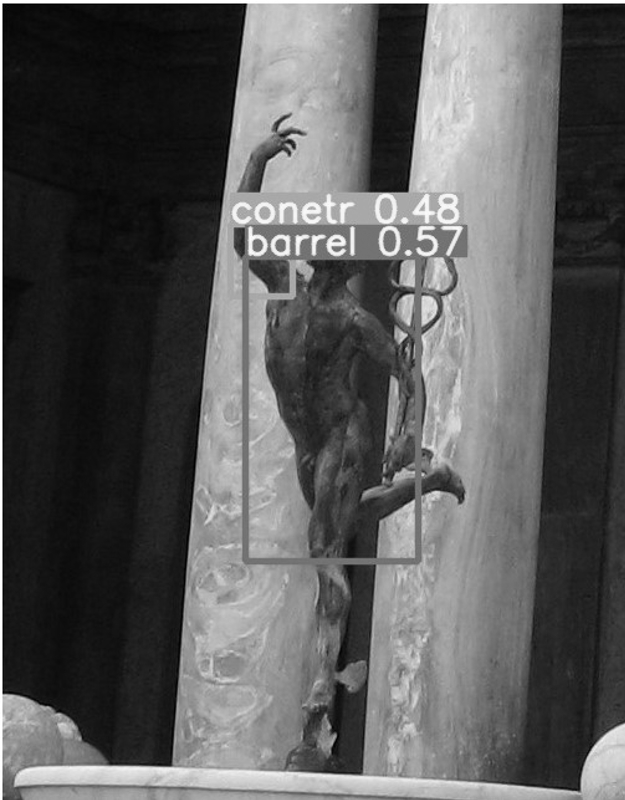}
  \includegraphics[scale=0.5]{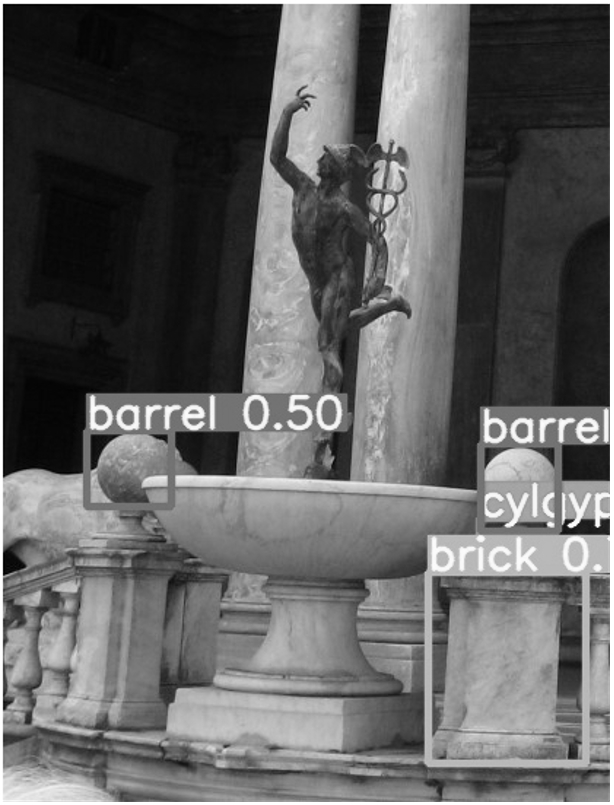}
  \includegraphics[scale=0.7]{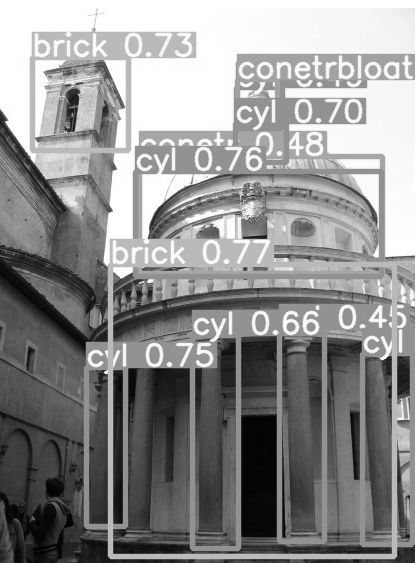}
  \includegraphics[scale=0.7]{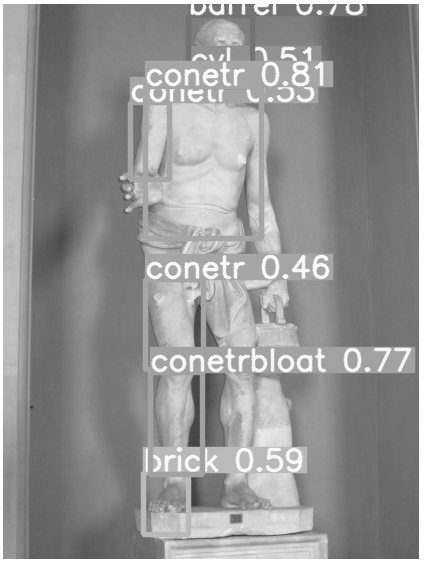}
\end{center}
\caption{An example of a complicated image decomposition in the case of
manually marked up images in the training set.}
\label{fig:10}       
\end{figure}
\begin{figure}\begin{center}
   \includegraphics[scale=0.6]{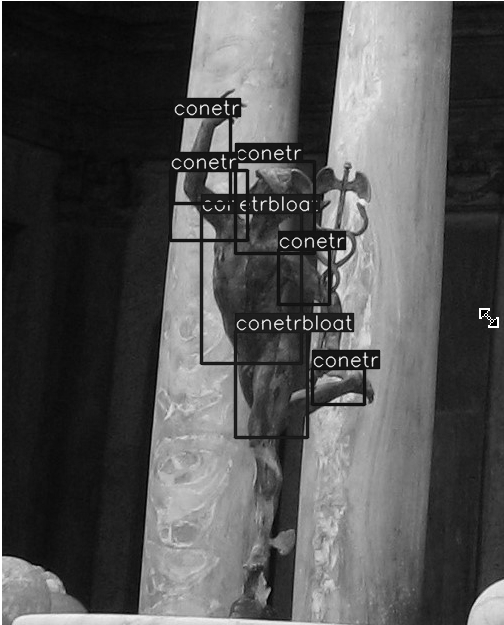}
  \includegraphics[scale=0.7]{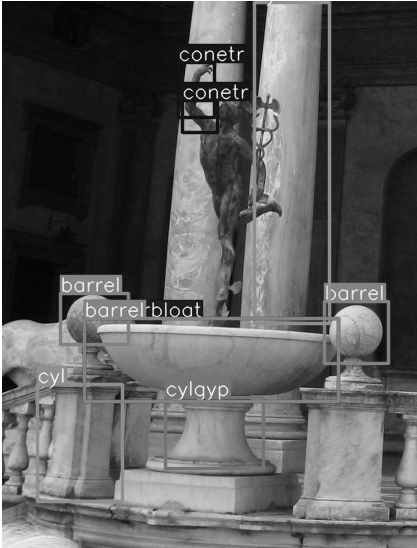}
  \includegraphics[scale=0.7]{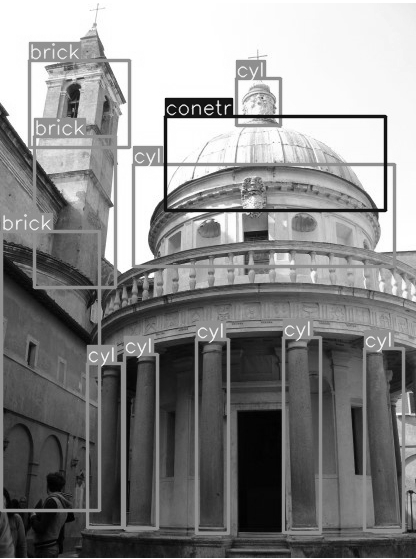}
  \includegraphics[scale=0.7]{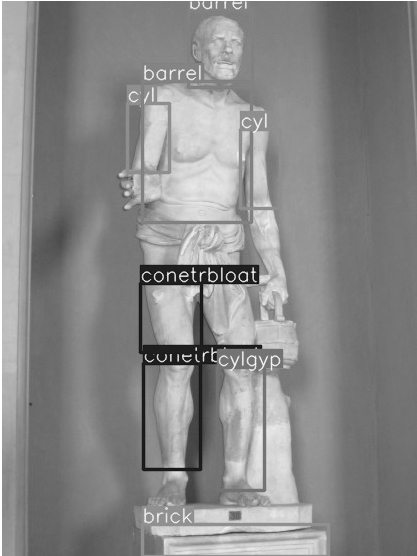}
\end{center}
\caption{An example of a complicated image decomposition in the case of automatically generated
 and manually marked up images in the training set.}
\label{fig:11}       
\end{figure}
Here, you can see that only an automatically generated image set is not enough for good image recognition under limited training: we use three series of 150 images with 10, 20 and 30 geons respectively, with 10x augmentation (resource \cite{alb} was used). The manually marked up image set is also not good, and only the combination of automatically generated images with the hand-marked ones gives an acceptable result\footnote{Authors are grateful to K. Rusakov for this advice} --- geon recognition becomes more exact and complete. We use here about 150 hand-marked images with the same augmentation\footnote{An image is not presented in these image sets if it was not recognized in the corresponding case}.

\section{System Movement Discussion  and Reward Valuations}\label{rew_val}
Formula \ref{eqmain} is beautiful since it gives a mathematical justification for generally obvious behaviour. However, it is necessary to clarify some points to use the formula in practice. Obviously, the system should move in such a manner as to increase all the rewards $\neg k^{A}(a)$ and $k^{B_{i}}(b_{j})$ of possible number of maximal goals'. Thus, what amounts should we take as the objects' rewards? Clearly, the number of visible objects' geons may not increase as the system gets closer to the object. This number may even decrease with the foreshortening change. However, the information amount increases since we recognize the object more confidently and from different sides.

Hence, we use the set $\cup \{b_{j}\}$ of geons in geon schemes of goals to be achieved as the anti-reward $\neg k^{A}(a)$ for the free move. The larger the set, the higher the reward of the non-free move and the lower the reward of the free move. For $k^{B_{i}}(b_{j})$ in the process of achieving goals, we use successively two reward types :

--- In the first type, the system approaches the object and sees it better and better. At this stage, we use the same anti-reward as in the case of a known object: the anti-reward is the set of possible recognition variants. When the system sees all visible geons $\{b_{j}\}$ of the thing $b_{j}$ precisely, without any hesitation, then the anti-reward is minimal, and the reward $k(p_{i},b_{j})$ in position $p_{i}$ is maximal. Thus, the system should move in such a manner as to decrease the anti-reward set;

--- In the second one, the system sees everything well. In this case, we take the union $\bigcup_{\substack{p_{k}\\k\leqslant i}} \{b_{j}\}$ of geons' sets (remembered up to the position $p_{i}$) of the object $b_{j}$ visible from different camera angles, as the reward $k(p_{i},b_{j})$. The geon sets are chosen to be localized in one place. This reward is maximal when the system has investigated the object from all sides and cannot get more information (i.e., cannot add new geons to the set corresponding to the item). Such a reward definition may demand external methods to decide that the object is investigated from all sides.

Thus, formula \ref{eqmain} mathematically justifies the apparent general behaviour of the system, but not a specific trajectory.

\section{Conclusion}\label{concl}
In this paper, we give the method of evaluating visual information about rewards along a system (e.g., a robot) path, the method of object recognition around the system (robot) and the description of a study process of an environment. We decompose visible objects on sets of geometric primitives (geons) which the system's neural network is pre-trained to recognize. The full combinations of these geons (geons' schemes) corresponding to unknown objects can be stored in the system's memory as new items. During an object investigation, the numbers of geons in such schemes are increased, and they are taken as the position rewards. Thus, the system should move in such a way as to put the highest possible number of geons connected with the object into its database. Later, we suppose to represent in detail the way to generate 3-D objects from the geons' schemes to train the neural network to recognize them without the database.

The calculation on the limited training database gives quite satisfactory geon recognition when the database includes, not only automatically generated images, but also a number of hand-marked ones.

Thus, we consider such a  cognition process of the environment by an intelligent system as a movement in the environment. We have used the intuition of baby-like behaviour during an object investigation in an environment to model a robot-like system to behave. The system movement was represented as a game in the definite game category in which the environment corresponds to the Opponent. He provides the Proponent (the system) with some information about the environment objects which may also be considered as the system goals. Different goals achieving is considered as parallel processes which are represented as the tensor product of correspondent games, and form a comprehensive game.

The game has rewards on its positions (geon sets), which estimate the quantity of the information provided by the environment. We demand the greatest total reward along the system play to choose the path. When the total reward of all parallel processes cannot be improved any more, we may decrease the number of selected goals' achieving processes up to, possibly, the one in the end. An external method of the path calculation can change this algorithm to create an objects' traversal path of the maximum cumulative reward getting.

We choose those goals' achieving processes from all those possible, which have the highest estimations in the system goal lattice. It is so because every goal and the corresponding process of achieving it have a definite correspondent truth value in the lattice. The higher the value lies in the lattice diagram; the higher the priority of the process is. The partial order on the lattice is generated by attention degrees of the system to its goals.

Thus, we consider two types of lattice estimations: the goal lattice value determines the choice of the goal achieving processes from all possible ones, and the position rewards of the game (geon sets) determine the optimal path of these chosen processes in the environment.

Such an approach may be useful in creating intelligent robotic systems. In these systems, ``brains'' can be brought from one robot to another one. However, the first such system must go through the
whole learning process as a real person, from childhood to adulthood. During the process, it must
explore objects of the outside world in themselves, recognize and classify them and evaluate their usefulness based on the goals of the system. Our approach offers one of the steps in this direction.


Such a model corresponds to the approach in which the system intelligence is considered as the consequence of the system predestination or purpose for practical
use, i.e., of the system goal lattice. In simple systems, i.e., ants, such preference structures may be pre-existent. In more complicated ones, these structures may be being built and changed during the system life.

\section{Declarations}
\begin{itemize}
\item Funding {'Not applicable'}
\item Conflict of interest/Competing interests  {'The authors declare that they have no known competing financial interests or personal relationships that could have appeared to influence the work reported in this paper'}
\item Availability of data and material {'Not applicable}
\item Code availability {'Not applicable}
\item Authors' contributions {'First author: the method of describing and assessing the process of exploration of an environment, the method of evaluating the amount of visual information about the object, YOLO training. Second author: the method of procedural generating of training images and corresponding programme'}
\end{itemize}

\bibliography{Maximov_Recognition}

\end{document}